\pdfoutput=1
\documentclass[11pt]{article}
\usepackage{acl2013}
\makeatletter
\newcommand{\@BIBLABEL}{\@emptybiblabel}
\newcommand{\@emptybiblabel}[1]{}
\makeatother
\usepackage{times}
\usepackage{url}
\usepackage{latexsym}
\usepackage{amsmath}  
\usepackage{verbatim} 
\usepackage{fancyhdr} 
\usepackage{hyperref}

\newtheorem{theorem}{Theorem}[section]
\newtheorem{lemma}[theorem]{Lemma}

\newenvironment{proof}[1][Proof]{\begin{trivlist}
\item[\hskip \labelsep {\bfseries #1}]}{\end{trivlist}}

\newcommand{\qed}{\nobreak \ifvmode \relax \else
      \ifdim\lastskip<1.5em \hskip-\lastskip
      \hskip1.5em plus0em minus0.5em \fi \nobreak
      \vrule height0.75em width0.5em depth0.25em\fi}

\title{Analysis of Stopping Active Learning based on Stabilizing Predictions}

\author{Michael Bloodgood \\
  Center for Advanced Study of Language \\
  University of Maryland \\
  College Park, MD 20740 \\
  {\tt meb@umd.edu} \\\And
  John Grothendieck \\
  Raytheon BBN Technologies \\
  9861 Broken Land Parkway, Suite 400 \\
  Columbia, MD 21046 \\
  {\tt jgrothen@bbn.com} \\}

\date{}

\begin{document}

\thispagestyle{fancy}

\maketitle
\begin{abstract}
Within the natural language processing (NLP) community, active learning has been widely investigated and applied in order to alleviate the annotation bottleneck faced by developers of new NLP systems and technologies.
This paper presents the first theoretical analysis of stopping active learning based on stabilizing predictions (SP). 
The analysis has revealed three elements that are central to the success of the SP method: 
(1) bounds on Cohen's Kappa agreement between successively trained models impose bounds on differences in F-measure performance of the models; 
(2) since the stop set does not have to be labeled, it can be made large in practice, helping to guarantee that the results transfer to previously unseen streams of examples 
at test/application time; and 
(3) good (low variance) sample estimates of Kappa between successive models can be obtained.
Proofs of relationships between the level of Kappa agreement and the difference in performance between consecutive models are presented. 
Specifically, if the Kappa agreement between two models exceeds a threshold T (where $T>0$), then the difference in F-measure performance 
between those models is bounded above by $\frac{4(1-T)}{T}$ in all cases. 
If precision of the positive conjunction of the models is assumed to be $p$, then the bound can be tightened to $\frac{4(1-T)}{(p+1)T}$. 
  
\end{abstract}

\section{Introduction}

{\em Active learning} (AL), also called {\em query learning} and {\em selective sampling}, is an approach to reduce
the costs of creating training data that has received considerable 
interest (e.g., 
\cite{argamon-engelson1999,baldridge2008,bloodgood2009a,bloodgood2010,hachey2005,haertel2008,haffari2009b,hwa2000,lewis1994,sassano2002,settles2008,shen2004,thompson1999,tomanek2007,zhu2007}). 

Within the NLP community, active learning has been widely investigated and applied in order to alleviate the annotation bottleneck faced by developers of new NLP systems and technologies.
The main idea is that by
judiciously selecting which examples to have labeled, annotation effort will be focused on the most 
helpful examples and less annotation effort will be required to achieve given levels of performance than 
if a passive learning policy had been used.

Historically, the problem of developing methods for detecting when to stop AL was tabled for future work and the research literature was focused on how to select
which examples to have labeled and analyzing the selection methods \cite{cohn1996,seung1992,freund1997,roy2001}.
However, to realize the savings in annotation effort that AL enables, we must have a method for knowing when to stop the annotation process.  
The challenge is that if we stop too early while useful generalizations are still being made, 
then we can wind up with a model that performs poorly, but if we stop too late after all the useful generalizations are made, then
human annotation effort is wasted and the benefits of using active learning are lost. 

Recently research has begun to develop methods for stopping AL 
\cite{schohn2000,ertekin2007a,ertekin2007b,zhu2007,laws2008,zhu2008a,zhu2008b,vlachos2008,bloodgood2009c,bloodgood2009b,ghayoomi2010}. 
The methods are all heuristics based on estimates of model confidence, error, or stability. 
Although these heuristic methods have appealing intuitions and have had experimental success on a small handful of tasks and datasets, the methods are not widely usable in practice yet because our community's understanding of the stopping methods remains too coarse and inexact.  
Pushing forward on understanding the mechanics of stopping at a more exact level is therefore crucial for achieving the design of widely usable effective stopping criteria.

\newcite{bloodgood2009b} introduce the terminology {\em aggressive} and {\em conservative} to describe the behavior of stopping 
methods\footnote{Aggressive methods stop sooner, aggressively trying to reduce unnecessary annotations while conservative methods are careful not
to risk losing model performance, even if it means annotating many more examples than were necessary.} and
conduct an empirical evaluation of the different published stopping methods on several datasets.  
While most stopping methods tend to behave conservatively, stopping based on stabilizing predictions computed via inter-model Kappa agreement
has been shown to be consistently aggressive without losing 
performance (in terms of F-Measure\footnote{For the rest of this paper, we will use F-measure to denote F1-measure, that is, the balanced 
harmonic mean of precision and recall, which is a standard metric used to evaluate NLP systems.}) in several published empirical tests.  
This method stops when the Kappa agreement between consecutively
learned models during AL exceeds a threshold for three consecutive iterations of AL.
Although this is an intuitive heuristic that has performed well in published experimental results, there has not been any theoretical analysis of the method.

The current paper presents the first theoretical analysis of stopping based on stabilizing predictions. 
The analysis helps to explain at a deeper and more exact level {\em why} the method works as it does.
The results of the analysis help to characterize classes of problems where the method can be expected to work well and where (unmodified) it will not be expected to 
work as well. 
The theory is suggestive of modifications to improve the robustness of the stopping method for certain classes of problems.   
And perhaps most important, the approach that we use in our analysis provides an enabling framework for more precise analysis of stopping criteria and 
possibly other parts of the active learning decision space.

In addition, the information presented in this paper is useful for works that consider switching between 
different active learning strategies and operating regions such as \cite{baram2004,donmez2007,roth2008}. 
Knowing when to switch strategies, for example, is similar to the stopping problem and is another setting where detailed understanding of the variance 
of stabilization estimates and their link to performance ramifications is useful.
More exact understanding of the mechanics of stopping is also useful for applications of co-training \cite{blum1998}, 
and agreement-based co-training \cite{clark2003} in particular.
Finally, the proofs of the Theorems regarding the relationships between Cohen's Kappa statistic and 
F-measure may be of broader use in works that consider inter-annotator
agreement and its ramifications for performance appraisals, a topic that has been of long-standing interest 
in computational linguistics \cite{carletta1996,artstein2008}.

In the next section we summarize the stabilizing predictions (SP) stopping method.
Section~\ref{analysis} analyzes SP and Section~\ref{conclusions} concludes. 

\section{Stopping Active Learning based on Stabilizing Predictions} \label{spSummary}

The intuition behind the SP method is that the models learned during AL can be applied to a large representative set of unlabeled data called a {\em stop set} and 
when consecutively learned models have high agreement on their predictions for classifying the examples in the stop set, 
this indicates that it is time to stop \cite{bloodgood2009b,bloodgood2009c}. 
The active learning stopping strategy explicitly examined in \cite{bloodgood2009b} (after the
general form is discussed) is to calculate Cohen's Kappa agreement statistic
between consecutive rounds of active learning and stop once it is above
0.99 for three consecutive calculations.  

Since the Kappa statistic is an important aspect of this method, we now discuss some background regarding measuring agreement in general, and Cohen's Kappa in particular.
Measurement of agreement between human annotators has received significant attention and in that context, the drawbacks of using 
percentage agreement have been recognized \cite{artstein2008}. Alternative metrics have been proposed that take chance agreement into account.
\newcite{artstein2008} survey several agreement metrics. Most of the agreement metrics they discuss are of the form:
\begin{equation} \label{genericAgreement}
agreement = \frac{A_o - A_e}{1 - A_e},
\end{equation}
where $A_o =$ observed agreement, and $A_e =$ agreement expected by chance.
The different metrics differ in how they compute $A_e$. 
All the instances of usage of an agreement metric in this article will have two categories and two coders. 
The two categories are ``+1" and ``-1" 
and the two coders are the two consecutive models for which agreement is being measured. 

Cohen's Kappa statistic\footnote{We note that there are other agreement measures (beyond Cohen's Kappa) which could also be 
applicable to stopping based on stabilizing predictions, but an analysis of these is outside the scope of the current paper.} \cite{cohen1960} measures agreement expected by chance by modeling each coder (in our case model) with 
a separate distribution governing their likelihood of assigning a particular category. 
Formally, Kappa is defined by Equation~\ref{genericAgreement} with $A_e$ computed as follows:
\begin{equation} \label{expectedAgreement}
A_e = \sum_{k \in \{+1,-1\}} P(k|c_1) \cdot P(k|c_2),
\end{equation}
where each $c_i$ is one of the coders (in our case, models), and $P(k|c_i)$ is the probability that coder (model) $c_i$ labels an instance 
as being in category $k$. Kappa estimates the $P(k|c_i)$ in Equation~\ref{expectedAgreement} based on the proportion of 
observed instances that coder (model) $c_i$ labeled as being in category $k$. 

\section{Analysis} \label{analysis}

This section analyzes the SP stopping method. 
Section~\ref{variance} analyzes the variance of the estimator of Kappa that SP uses and in particular the relationship of this variance to specific aspects of the operationalization of SP, such as the stop set size.
Section~\ref{bounds} analyzes relationships between the Kappa agreement between two models and the difference in F-measure between those two models. 

\subsection{Variance of Kappa Estimator} \label{variance}

SP bases its decision to stop on the information contained in the contingency tables between the classifications of models learned at consecutive iterations during AL.
In determining whether to stop at iteration t, the classifications of the current model $M_t$ are compared with the classifications of the previous model $M_{t-1}$. 
Table~\ref{t:populationProbs} shows the population parameters for these two models, where: population probability $\pi_{ij}$ for $i,j \in \{+,-\}$ is the probability of an example being placed in 
category $i$ by model $M_{t-1}$ and category $j$ by model $M_t$; population probability $\pi_{.j}$ for $j \in \{+,-\}$ is the probability of an example being placed in category $j$ by model $M_t$; 
and population probability $\pi_{i.}$ for $i \in \{+,-\}$ is the probability of an example being placed in category $i$ by model $M_{t-1}$. 
The actual probability of agreement is $\pi_{o} = \pi_{++} + \pi_{--}$. As indicated in Equation~\ref{expectedAgreement}, Kappa models the probability of agreement expected due to chance by assuming that
classifications are made independently. 
Hence, the probability of agreement expected by chance in terms of the population probabilities is $\pi_{e} = \pi_{+.}\pi_{.+} + \pi_{-.}\pi_{.-}$. 
From the definition of Kappa (see Equation~\ref{genericAgreement}), we then have that the Kappa parameter $K$ in terms of the population probabilities is given by 
\begin{equation} \label{populationKappa}
K = \frac{\pi_o - \pi_e}{1 - \pi_e}.
\end{equation}

\begin{table}
\begin{center}
\begin{tabular}{cccc} \hline
          & \multicolumn{2}{c}{$M_{t}$} &            \\ \cline{2-3}
$M_{t-1}$ & +            & -             & Total      \\ \hline 
+         & $\pi_{++}$   & $\pi_{+-}$    & $\pi_{+.}$ \\ 
-         & $\pi_{-+}$   & $\pi_{--}$    & $\pi_{-.}$ \\ \hline
Total     & $\pi_{.+}$   & $\pi_{.-}$    & 1          \\ \hline
\end{tabular}
\end{center}
\caption{\label{t:populationProbs} Contingency table population probabilities for $M_t$ (model learned at iteration t) and $M_{t-1}$ (model learned at iteration t-1).}
\end{table}

For practical applications we will not know the true population probabilities and we will have to resort to using sample estimates.
The SP method uses a stop set of size $n$ for deriving its estimates. 
Table~\ref{t:sampleCounts} shows the contingency table counts for the classifications of models $M_t$ and $M_{t-1}$ on a sample of size $n$.
The population probabilities $\pi_{ij}$ can be estimated by the relative frequencies $p_{ij}$ for $i,j \in \{+,-,.\}$, where: $p_{++} = a/n$; $p_{+-} = b/n$; $p_{-+} = c/n$; $p_{--} = d/n$; 
$p_{+.} = (a+b)/n$; $p_{-.} = (c+d)/n$; $p_{.+} = (a+c)/n$; and $p_{.-} = (c+d)/n$. Let $p_o = p_{++} + p_{--}$, the observed proportion of agreement and let $p_e = p_{+.}p_{.+} + p_{-.}p_{.-}$, the 
proportion of agreement expected by chance if we assume that $M_t$ and $M_{t-1}$ make their classifications independently. Then the Kappa measure of agreement K between $M_t$ and $M_{t-1}$ 
(see Equation~\ref{populationKappa}) is estimated by 
\begin{equation} \label{sampleKappa}
\hat{K} = \frac{p_o - p_e}{1 - p_e}.
\end{equation}

\begin{table}
\begin{center}
\begin{tabular}{cccc} \hline
          & \multicolumn{2}{c}{$M_{t}$} &            \\ \cline{2-3}
$M_{t-1}$ & +            & -             & Total      \\ \hline 
+         & $a$     & $b$      & $a+b$ \\ 
-         & $c$     & $d$      & $c+d$ \\ \hline
Total     & $a+c$   & $b+d$    & $n$          \\ \hline
\end{tabular}
\end{center}
\caption{\label{t:sampleCounts} Contingency table counts for $M_t$ (model learned at iteration t) and $M_{t-1}$ (model learned at iteration t-1).}
\end{table}

Using the delta method, as described in \cite{bishop1975}, \newcite{fleiss1969} derived an estimator of the large-sample variance of $\hat{K}$.  
According to \newcite{hale1993}, the estimator simplifies to
 
\begin{equation} \label{sampleKappaVariance}
\begin{split}
& Var(\hat{K}) = \frac{1}{n(1-p_e)^2} \times \\
& \Bigg\{ \sum_{i \in \{+,-\}}{p_{ii}[1-4\bar{p}_i(1-\hat{K})]} \\
& -(\hat{K} - p_e(1-\hat{K}))^2 + (1-\hat{K})^2 \times \\
& \sum_{i,j \in \{+,-\}}{p_{ij}[2 (\bar{p}_i + \bar{p}_j) - (p_{i.} + p_{.j})]^2} \Bigg\}, 
\end{split}
\end{equation}  
where $\bar{p}_i = (p_{i.} + p_{.i})/2$. From Equation~\ref{sampleKappaVariance}, we can see that the variance of our estimate of Kappa is inversely proportional to the 
size of the stop set we use.

\newcite{bloodgood2009b} used a stop set of size 2000 for each of their datasets. 
Although this worked well in the results they reported, we do not believe that 2000 is a fixed size that will work well 
for all tasks and datasets where the SP method could be used. 
Table~\ref{t:kappaVariances} shows the variances of $\hat{K}$ computed 
using Equation~\ref{sampleKappaVariance} at the points at which SP stopped AL for each of 
the datasets\footnote{We note that each of the datasets was set up as a binary 
classification task (or multiple binary classification tasks). Further details and descriptions of each 
of the datasets can be found in \cite{bloodgood2009b}.} from \cite{bloodgood2009b}.

\begin{table*}[t]
\begin{center}
\begin{tabular}{|l|c|} \hline
Task-Dataset & Variance of $\hat{K}$ \\ \hline
NER-DNA (10-fold CV)          & 0.0000223 \\ \hline
NER-cellType (10-fold CV)     & 0.0000211 \\ \hline
NER-protein (10-fold CV)      & 0.0000074 \\ \hline
Reuters (10 Categories)       & 0.0000298 \\ \hline
20 Newsgroups (20 Categories) & 0.0000739 \\ \hline
WebKB Student (10-fold CV)    & 0.0000137 \\ \hline
WebKB Project (10-fold CV)    & 0.0000190 \\ \hline
WebKB Faculty (10-fold CV)    & 0.0000115 \\ \hline
WebKB Course (10-fold CV)     & 0.0000179 \\ \hline
TC-spamassassin (10-fold CV)  & 0.0000042 \\ \hline
TC-TREC-SPAM (10-fold CV)     & 0.0000043 \\ \hline
Average (macro-avg)           & 0.0000209 \\ \hline
\end{tabular}
\end{center}
\caption{\label{t:kappaVariances} Estimates of the variance of $\hat{K}$. For each dataset, the estimate of the variance of $\hat{K}$ is computed (using Equation~\ref{sampleKappaVariance}) from the 
contingency table at the point at which SP stopped AL and the average of all the variances (across all folds of CV) is displayed. 
The last row contains the macro-average of the average variances for all the datasets.}
\end{table*}

These variances indicate
that the size of 2000 was typically sufficient to get tight estimates of
Kappa, helping to illuminate the empirical success of the SP method on
these datasets. More generally, the SP method can be augmented with a
variance check: if the variance of estimated Kappa at a potential
stopping point exceeds some desired threshold, then the stop set size can be
increased as needed to reduce the variance.

Looking at Equation~\ref{sampleKappaVariance} again, one can note that when $p_e$ is relatively close to 1, the variance of $\hat{K}$ can be expected to get quite large. In these situations, users of SP should expect to have to use larger stop set sizes and in extreme conditions, SP may not be an advisable method to use.

\subsection{Relationship between Kappa agreement and change in performance between models} \label{bounds}
Heretofore, the published literature contained only informal explanations of why stabilizing predictions is expected to work well as a stopping method (along with 
empirical tests demonstrating successful operation on a handful of tasks and datasets).
In the remainder of this section we describe the mathematical foundations for stopping methods based on stabilizing predictions.  
In particular, we will prove that even in the worst possible case, if the Kappa agreement between two subsequently learned models is 
greater than a threshold $T$, 
then it must be the case that the change in performance between these two models is bounded above by $\frac{4(1-T)}{T}$.  
We then go on to prove additional Theorems that tighten this bound when assumptions are made about model precision.

\begin{lemma} \label{FGreaterK}
Suppose F-measure $F$ and Kappa $K$ are computed from the same contingency table of counts, such as the one given in Table~\ref{t:sampleCounts}. 
Suppose $ad-bc \ge 0$. Then $F \ge K$. 
\end{lemma}
\begin{proof}
By definition, in terms of the contingency table counts, 
\begin{equation} \label{KDef}
K = \frac{2ad-2bc}{(a+b)(b+d)+(a+c)(c+d)} 
\end{equation}
and 
\begin{equation} \label{FDef}
F = \frac{2a}{2a+b+c}. 
\end{equation}
Rewriting $F$ so that it will have the same numerator as $K$, we have: 
\begin{eqnarray}
F & = & F \Bigg( \frac{d-\frac{bc}{a}}{d-\frac{bc}{a}} \Bigg) \\
  & = & \Big( \frac{2a}{2a+b+c} \Big) \Bigg( \frac{d-\frac{bc}{a}}{d-\frac{bc}{a}} \Bigg) \\
  & = & \frac{2ad-2bc}{2ad+bd+cd-2bc-\frac{b^2c + bc^2}{a}}. \label{FSameNumeratorK} 
\end{eqnarray}
We can see that the expression for $F$ in Equation~\ref{FSameNumeratorK} has the same numerator as $K$ in Equation~\ref{KDef}  
but the denominator of $K$ in Equation~\ref{KDef} is $\ge$ the denominator of $F$ in Equation~\ref{FSameNumeratorK}. 
Therefore, $F \ge K$. \qed 
\end{proof}

\begin{theorem} \label{KBoundDeltaF}
Let $M_t$ be the model learned at iteration $t$ of active learning and $M_{t-1}$ be the model learned at iteration $t-1$. 
Let $K_t$ be the estimate of Kappa agreement between the classifications of $M_t$ and $M_{t-1}$ on the examples in the stop set.
Let $\tilde{F}_t$ be the F-measure between the classifications of $M_t$ and truth on the stop set. 
Let $\tilde{F}_{t-1}$ be the F-measure between the classifications of $M_{t-1}$ and truth on the stop set. 
Let $\Delta F_t$ be $\tilde{F}_t - \tilde{F}_{t-1}$. 
Suppose $T>0$.
Then $K_t > T \Rightarrow |\Delta F_t| \le \frac{4(1-T)}{T}$. 
\end{theorem}
\begin{proof}
Suppose $M_t$, $M_{t-1}$, $K_t$, $\tilde{F}_t$, $\tilde{F}_{t-1}$, $\Delta F_t$, and $T$ are defined as stated in the statement of Theorem~\ref{KBoundDeltaF}. 
Let $F_t$ be the F-measure between the classifications of $M_t$ and $M_{t-1}$ on the examples in the stop set. 
Let Table~\ref{t:sampleCounts} show the contingency table counts for $M_t$ versus $M_{t-1}$ on the examples in the stop set.
Then, from their definitions, we have $K_t = \frac{2(ad-bc)}{(a+b)(b+d)+(a+c)(c+d)}$ and $F_t = \frac{2a}{2a+b+c}$. 
There exist true labels for the examples in the stop set, which we don't know since the stop set is unlabeled, but nonetheless must exist. 
We use the truth on the stop set to split Table~\ref{t:sampleCounts} into two subtables of counts, one table for all the examples that are truly positive and one table for all the examples that are truly negative.
Table~\ref{t:positives} shows the contingency table for $M_t$ versus $M_{t-1}$ for all of the examples in the stop set that have true labels of +1 and  
Table~\ref{t:negatives} shows the contingency table for $M_t$ versus $M_{t-1}$ for all of the examples in the stop set that have true labels of -1. 
\begin{table}
\begin{center}
\begin{tabular}{cccc} \hline
          & \multicolumn{2}{c}{$M_{t}$} &            \\ \cline{2-3}
$M_{t-1}$ & +            & -             & Total      \\ \hline 
+         & $a_1$     & $b_1$      & $a_1+b_1$ \\ 
-         & $c_1$     & $d_1$      & $c_1+d_1$ \\ \hline
Total     & $a_1+c_1$   & $b_1+d_1$    & $n_1$          \\ \hline
\end{tabular}
\end{center}
\caption{\label{t:positives} Contingency table counts for $M_t$ (model learned at iteration t) versus $M_{t-1}$ (model learned at iteration t-1) for only the examples in the stop set that have truth = +1.}
\end{table}
\begin{table}
\begin{center}
\begin{tabular}{cccc} \hline
          & \multicolumn{2}{c}{$M_{t}$} &            \\ \cline{2-3}
$M_{t-1}$ & +            & -             & Total      \\ \hline 
+         & $a_{-1}$     & $b_{-1}$      & $a_{-1}+b_{-1}$ \\ 
-         & $c_{-1}$     & $d_{-1}$      & $c_{-1}+d_{-1}$ \\ \hline
Total     & $a_{-1}+c_{-1}$   & $b_{-1}+d_{-1}$    & $n_{-1}$          \\ \hline
\end{tabular}
\end{center}
\caption{\label{t:negatives} Contingency table counts for $M_t$ (model learned at iteration t) versus $M_{t-1}$ (model learned at iteration t-1) for only the examples in the stop set that have truth = -1.}
\end{table}

From Tables~\ref{t:sampleCounts}, \ref{t:positives}, and \ref{t:negatives} one can see that $a$ is the number of examples in the stop set that both $M_t$ and $M_{t-1}$ classified as positive. 
Furthermore, out of these $a$ examples, $a_1$ of them truly are positive and $a_{-1}$ of them truly are negative. Similar explanations hold for the other counts.
Also, from Tables~\ref{t:sampleCounts}, \ref{t:positives}, and \ref{t:negatives}, one can see that the equalities $a=a_1+a_{-1}$, $b=b_1+b_{-1}$, $c=c_1+c_{-1}$, and $d=d_1+d_{-1}$ all hold. 
The contingency tables for $M_t$ versus truth and $M_{t-1}$ versus truth can be derived from Tables~\ref{t:positives} and \ref{t:negatives}. 
For convenience, Table~\ref{t:MtVersusTruth} shows the contingency table for $M_t$ versus truth 
and Table~\ref{t:Mt-1VersusTruth} shows the contingency table for $M_{t-1}$ versus truth. 
\begin{table}
\begin{center}
\begin{tabular}{cccc} \hline
          & \multicolumn{2}{c}{$M_{t}$}  &            \\ \cline{2-3}
Truth     & +               & -               & Total      \\ \hline 
+         & $a_1+c_1$       & $b_1+d_1$       & $n_1$ \\ 
-         & $a_{-1}+c_{-1}$ & $b_{-1}+d_{-1}$ & $n_{-1}$ \\ \hline
Total     & $a+c$           & $b+d$           & $n$          \\ \hline
\end{tabular}
\end{center}
\caption{\label{t:MtVersusTruth} Contingency table counts for $M_t$ (model learned at iteration t) versus truth. (Derived from Tables~\ref{t:positives} and \ref{t:negatives}}
\end{table}
\begin{table}
\begin{center}
\begin{tabular}{cccc} \hline
          & \multicolumn{2}{c}{$M_{t-1}$}  &            \\ \cline{2-3}
Truth     & +               & -               & Total      \\ \hline 
+         & $a_1+b_1$       & $c_1+d_1$       & $n_1$ \\ 
-         & $a_{-1}+b_{-1}$ & $c_{-1}+d_{-1}$ & $n_{-1}$ \\ \hline
Total     & $a+b$           & $c+d$           & $n$          \\ \hline
\end{tabular}
\end{center}
\caption{\label{t:Mt-1VersusTruth} Contingency table counts for $M_{t-1}$ (model learned at iteration t-1) versus truth. (Derived from Tables~\ref{t:positives} and \ref{t:negatives}}
\end{table}
Suppose that $K_t > T$. This implies, by Lemma~\ref{FGreaterK}\footnote{Note that the condition $ad-bc \ge 0$ of Lemma~\ref{FGreaterK} is met since $K_t > T$ and $T > 0$ imply
$K_t > 0$, which in turn implies $ad - bc > 0$.}, that $F_t > T$. 
This implies that
\begin{eqnarray}
            & \frac{2a}{2a+b+c} > T  \\
\Rightarrow & 2a > (2a+b+c)T \label{2aGreaterThan2aPlusbPluscT}  \\
\Rightarrow & 2a(1-T) > (b+c)T  \\
\Rightarrow & b+c < \frac{2a(1-T)}{T}. \label{bPluscBound} 
\end{eqnarray}
Note that Equations~\ref{2aGreaterThan2aPlusbPluscT} and \ref{bPluscBound} are justified since $2a+b+c>0$ and $T>0$, respectively.

From Table~\ref{t:MtVersusTruth} we can see that ${\tilde{F}_t = \frac{2(a_1+c_1)}{2(a_1+c_1)+b_1+d_1+a_{-1}+c_{-1}}}$; 
from Table~\ref{t:Mt-1VersusTruth} we can see that $\tilde{F}_{t-1} = \frac{2(a_1+b_1)}{2(a_1+b_1)+c_1+d_1+a_{-1}+b_{-1}}$. 
For notational convenience, let: $g = 2(a_1+c_1)+b_1+d_1+a_{-1}+c_{-1}$; and $h = 2(a_1+b_1)+c_1+d_1+a_{-1}+b_{-1}$.

It follows that 
\begin{align} 
\Delta F_t &=  \frac{2(a_1+c_1)}{g} - \frac{2(a_1+b_1)}{h} \\ 
&= \frac{(2a_1+2c_1)h - (2a_1+2b_1)g}{gh} \label{part1}
\end{align} 

For notational convenience, let: $x=2(a_1c_1+a_1b_{-1}+c_1^2+c_1d_1+c_1a_{-1}+c_1b_{-1})$; and $y=2(a_1b_1+a_1c_{-1}+b_1^2+b_1d_1+b_1a_{-1}+b_1c_{-1})$.
Then picking up from Equation~\ref{part1}, it follows that
\begin{align}
\Delta F_t &= \frac{x - y}{gh} \\
&=  \frac{2[u_1+c_1u_2-b_1u_3]}{gh},
\end{align}
where $u_1=a_1c_1-a_1b_1+a_1b_{-1}-a_1c_{-1}$, $u_2=c_1+d_1+a_{-1}+b_{-1}$, and $u_3=b_1+d_1+a_{-1}+c_{-1}$.

For notational convenience, let: $d_A = c_1-b_1$ and 
$d_B = c_{-1}-b_{-1}$.
Then it follows that
\begin{equation}
\Delta F_t = \frac{2u_4}{gh},
\end{equation}
where: $u_4=a_1(d_A-d_B)+d_A(d_1+a_{-1}+b_1+c_1)+c_1b_{-1}-b_1c_{-1}$.

Noting that $g = h+d_A+d_B$, we have
\begin{equation}
\Delta F_t = \frac{2u_4}{h(h+d_A+d_B)}.
\end{equation}

Noting that $2u_4 = 2[d_A(a_1+b_1+c_1+d_1+a_{-1}+b_{-1})-d_B(a_1+b_1)]$ and letting $u_5 = a_1+b_1+c_1+d_1+a_{-1}+b_{-1}$, we have
\begin{equation}
\Delta F_t = \frac{2[d_Au_5-d_B(a_1+b_1)]}{h(h+d_A+d_B)}.
\end{equation}

Therefore, 
\begin{equation} \label{worstPossibleCase}
\begin{split}
&|\Delta F_t| \le 2\Bigg(\left|\frac{d_Au_5}{h(h+d_A+d_B)}\right| \\
&+ \left|\frac{d_B(a_1+b_1)}{h(h+d_A+d_B)}\right|\Bigg)
\end{split}
\end{equation}

Recall that $b+c=b_1+b_{-1}+c_1+c_{-1}$. Then observe that the following three inequalities hold: $b+c \ge d_A$; $b+c \ge d_B$; and $h(h+d_A+d_B) > 0$. 
Therefore, 
\begin{eqnarray}
|\Delta F_t| \le & \frac{2(b+c)[2a_1+2b_1+c_1+d_1+a_{-1}+b_{-1}]}{h(h+d_A+d_B)} \label{bPluscLoose} \\
= & \frac{2(b+c)h}{h(h+d_A+d_B)}  \\
= & \frac{2(b+c)}{h+d_A+d_B}  \\
\le & \frac{2(2a)(1-T)}{T(h+d_A+d_B)} \label{usingbPluscBound} \\
= & \big( \frac{4(1-T)}{T} \big) \big( \frac{a}{h+d_A+d_B} \big).  \label{rightBeforeFinalResult} 
\end{eqnarray}
Observe that $h+d_A+d_B = 2a_1+b_1+2c_1+d_1+a_{-1}+c_{-1}$. Therefore, $\frac{a}{h+d_A+d_B} \le 1$. 
Therefore, we have
\begin{equation} \label{finalResult}
|\Delta F_t| \le \frac{4(1-T)}{T}.  \; \qed
\end{equation} 
\end{proof}

Note that in deriving Inequality~\ref{usingbPluscBound}, we used the previously derived Inequality~\ref{bPluscBound}. 
Also, the proof of Theorem~\ref{KBoundDeltaF} assumes a worst possible case in the sense that all examples where the classifications of $M_t$ and $M_{t-1}$ differ 
are assumed to have truth values that all serve to maximize one model's F-measure and 
minimize the other model's F-measure so as to maximize $|\Delta F_t|$ as much as possible.
A resulting limitation is that the bound is loose in many cases. 
It may be possible to derive tighter bounds, perhaps by easing off to an expected case instead of a worst case 
and/or by making additional assumptions.\footnote{If one is planning to undertake this challenge, we would suggest 
further consideration of Inequalities~\ref{worstPossibleCase}, \ref{bPluscLoose}, \ref{usingbPluscBound}, and \ref{finalResult} as a 
possible starting point.} 

Taking this possibility up, we now prove tighter bounds when assumptions about the precision of the models $M_t$ and $M_{t-1}$ are made. 
Consider that in the proof of Theorem~\ref{KBoundDeltaF} when transitioning from Equality~\ref{rightBeforeFinalResult} to Inequality~\ref{finalResult}, we used 
the fact that $\frac{a}{h+d_A+d_B} \le 1$. Note that $\frac{a}{h+d_A+d_B} = \frac{a}{2a_1+b_1+2c_1+d_1+a_{-1}+c_{-1}}$, from which one sees 
that $\frac{a}{h+d_A+d_B} = 1$ only if all of $a_1,b_1,c_1,d_1$ and $c_{-1}$ are all zero. This is a pathological case. 
In many practically important classes of cases to consider, $\frac{a}{h+d_A+d_B}$ will be strictly less than $1$, and often substantially less than $1$. 
The following two Theorems prove tighter bounds on $|\Delta F_t|$ than Theorem~\ref{KBoundDeltaF} by utilizing this insight.
 
\begin{theorem} \label{KBoundDeltaFAssumingPerfectPrecision}
Suppose $M_t$, $M_{t-1}$, $K_t$, $\tilde{F}_t$, $\tilde{F}_{t-1}$, $\Delta F_t$, and $T$ are defined as stated in the statement of Theorem~\ref{KBoundDeltaF}. 
Let the contingency tables be defined as they were in the proof of Theorem~\ref{KBoundDeltaFAssumingPerfectPrecision}. 
Let $M_{PositiveConjunction}$ be a model that only classifies an example as positive if both models $M_t$ and $M_{t-1}$ classify the example as positive. 
Suppose that $M_{PositiveConjunction}$ has perfect precision on the stop set, or in other words that every single example from the stop set 
that both $M_t$ and $M_{t-1}$ classify as positive is truthfully positive (i.e., $a_{-1}=0$). 
Then $K_t > T \Rightarrow |\Delta F_t| \le \frac{2(1-T)}{T}$. 
\end{theorem}
\begin{proof}
The proof of Theorem~\ref{KBoundDeltaF} holds exactly as it is up until Equality~\ref{rightBeforeFinalResult}. 
Now, using the additional assumption that $a_{-1}=0$, we have $\frac{a}{h+d_A+d_B} \le \frac{1}{2}$. 
Therefore, we have 
\begin{equation} \label{finalResultAssumingPerfectPrecision}
|\Delta F_t| \le \frac{2(1-T)}{T}.  \; \qed
\end{equation} 
\end{proof}

Theorem~\ref{KBoundDeltaFAssumingPerfectPrecision} is a special case (in the limit) of a more general Theorem. 
Before stating and proving the more general Theorem, we prove a Lemma that will be helpful in making the proof of the general Theorem clearer. 

\begin{lemma} \label{helperLemma}
Let $f$, $d_A$, $d_B$ and contingency table counts be defined as they were in the proof of Theorem~\ref{KBoundDeltaF}. 
Suppose $a_1 = xa_{-1}$. Then $\frac{a}{h+d_A+d_B} \le \frac{x+1}{2x+1}$.
\end{lemma}
\begin{proof}
$a_1 = xa_{-1}$ by hypothesis. $a = a_1+a_{-1}$ by definition of contingency table counts. 
Hence, $a = (x+1)a_{-1}$. Therefore, 
\begin{eqnarray}
\frac{a}{h+d_A+d_B} \le & \frac{(x+1)a_{-1}}{2xa_{-1}+a_{-1}} \\ \nonumber
= & \frac{(x+1)a_{-1}}{(2x+1)a_{-1}} \\ \nonumber
= & \frac{x+1}{2x+1}. \qed
\end{eqnarray}
\end{proof}

The following Theorem generalizes Theorem~\ref{KBoundDeltaFAssumingPerfectPrecision} to cases when $M_{PositiveConjunction}$ has 
precision $p$ in $(0,1)$.\footnote{The case when $p=0$ is handled by Theorem~\ref{KBoundDeltaF} and the case when $p=1$ is handled 
by Theorem~\ref{KBoundDeltaFAssumingPerfectPrecision}.}

\begin{theorem} \label{KBoundDeltaFAssumingPrecisionP}
Suppose $M_t$, $M_{t-1}$, $K_t$, $\tilde{F}_t$, $\tilde{F}_{t-1}$, $\Delta F_t$, and $T$ are defined as stated in the statement of Theorem~\ref{KBoundDeltaF}. 
Let the contingency tables be defined as they were in the proof of Theorem~\ref{KBoundDeltaF}. 
Let $M_{PositiveConjunction}$ be a model that only classifies an example as positive if both models $M_t$ and $M_{t-1}$ classify the example as positive. 
Suppose that $M_{PositiveConjunction}$ has precision p on the stop set. 
Then $K_t > T \Rightarrow |\Delta F_t| \le \frac{4(1-T)}{(p+1)T}$. 
\end{theorem}
\begin{proof}
The proof of Theorem~\ref{KBoundDeltaF} holds exactly as it is up until Equality~\ref{rightBeforeFinalResult}.  
$M_{PositiveConjunction}$ has precision $p$ on the stop set $\Rightarrow p = \frac{a_1}{a_1+a_{-1}}$.  
Solving for $a_1$ in terms of $a_{-1}$ we have $a_1 = \frac{p}{1-p}a_{-1}$. 
Therefore, applying Lemma~\ref{helperLemma} with $x=\frac{p}{1-p}$, we have $\frac{a}{h+d_A+d_B} \le \frac{\frac{p}{1-p}+1}{\frac{2p}{1-p}+1}$. 
Therefore we have 
\begin{eqnarray} 
|\Delta F_t| \le & 4 \Bigg(\frac{\frac{p}{1-p}+1}{\frac{2p}{1-p}+1} \Bigg)\frac{(1-T)}{T} \\ 
= & \frac{4(1-T)}{(p+1)T}. \label{finalResultAssumingPrecisionP}  \; \qed
\end{eqnarray} 
\end{proof}

The scaling factor $\frac{1}{p+1}$ in Theorem~\ref{KBoundDeltaFAssumingPrecisionP} shows how the precision of the conjunctive model affects the bound. 
Theorem~\ref{KBoundDeltaF} had the scaling factor implicitly set to 1 in order to handle the pathological case where the positive conjunctive model has precision = 0. 
In Theorem~\ref{KBoundDeltaFAssumingPerfectPrecision}, where the positive conjunctive model has precision = 1 on the examples in the stop set, the 
scaling factor is set to 1/2. 
Theorem~\ref{KBoundDeltaFAssumingPrecisionP} generalizes the scaling factor so that it is a function of the precision of the positive conjunctive model. 
For convenience, Table~\ref{t:scalingFactorValues} shows the scaling factor values for a few different precision values.

\begin{table}
\begin{center}
\begin{tabular}{|c|c|c|} \hline
Precision & $\frac{1}{p+1}$ (to 3 decimal places)\\ \hline
50\%      & 0.667 \\ \hline
80\%      & 0.556 \\ \hline
90\%      & 0.526 \\ \hline
95\%      & 0.513 \\ \hline
98\%      & 0.505 \\ \hline
99\%      & 0.503 \\ \hline
99.9\%    & 0.500 \\ \hline
\end{tabular}
\end{center}
\caption{\label{t:scalingFactorValues} Values of the scaling factor from Theorem~\ref{KBoundDeltaFAssumingPrecisionP} for different precision values. }
\end{table}

The bounds in Theorems~\ref{KBoundDeltaF}, \ref{KBoundDeltaFAssumingPerfectPrecision}, and \ref{KBoundDeltaFAssumingPrecisionP} all bound the difference 
in performance \emph{on the stop set} of two consecutively learned models $M_t$ and $M_{t-1}$. 
An issue to consider is how connected the difference in performance on the stop set is to the difference in performance on a stream of application examples
generated according to the population probabilities. 
Taking up this issue, consider that the proof of Theorems~\ref{KBoundDeltaF}, \ref{KBoundDeltaFAssumingPerfectPrecision}, 
and \ref{KBoundDeltaFAssumingPrecisionP} would hold as it is if we had used sample proportions instead of sample counts 
(this can be seen by simply dividing every count by $n$, the size of the stop set).  
Since the stop set is unbiased (selected at random from the population), as $n$ approaches infinity, the sample proportions will approach the population
probabilities and the difference between the difference in performance between $M_t$ and $M_{t-1}$ on the stop set and on a stream of application examples
generated according to the population probabilities will approach zero.  

\section{Conclusions} \label{conclusions}

To date, the work on stopping criteria has been dominated by heuristics based on intuitions and experimental success on a small handful of tasks and datasets. 
But the methods are not widely usable in practice yet because our community's understanding of the stopping methods remains too inexact.  
Pushing forward on understanding the mechanics of stopping at a more exact level is therefore crucial for achieving the design of widely usable effective stopping criteria.

This paper presented the first theoretical analysis of stopping based on stabilizing predictions.  The analysis revealed three elements that are central to the 
SP method's success: 
(1) the sample estimates of Kappa have low variance; 
(2) Kappa has tight connections with differences in F-measure; and 
(3) since the stop set doesn't have to be labeled, it can be arbitrarily large, helping to guarantee that the results transfer to
previously unseen streams of examples at test/application time. 

We presented proofs of relationships between the level of Kappa agreement and the difference in performance between consecutive models. 
Specifically, if the Kappa agreement between two models is at least T, then the difference in F-measure performance between those 
models is bounded above by $\frac{4(1-T)}{T}$. 
If precision of the positive conjunction of the models is assumed to be $p$, then the bound can be tightened to $\frac{4(1-T)}{(p+1)T}$. 

The setup and methodology of the proofs can serve as a launching pad for many further investigations, including: 
analyses of stopping; works that consider switching between different active learning strategies and operating regions; 
and works that consider stopping co-training, and especially agreement-based co-training.
Finally, the relationships that have been exposed between the Kappa statistic and F-measure may be of broader use in works that consider inter-annotator
agreement and its interplay with system evaluation, a topic that has been of long-standing interest. 

\bibliographystyle{acl}
\bibliography{paper}

\end{document}